\documentclass[conference]{IEEEtran}
\IEEEoverridecommandlockouts
\usepackage{cite}
\usepackage{amsmath,amssymb,amsfonts}
\usepackage{algorithmic}
\usepackage{graphicx}
\usepackage{textcomp}
\usepackage{xcolor}
\def\BibTeX{{\rm B\kern-.05em{\sc i\kern-.025em b}\kern-.08em
    T\kern-.1667em\lower.7ex\hbox{E}\kern-.125emX}}

\usepackage{algorithm}
\usepackage{algorithmic}
\usepackage{amsthm}

\newtheorem{mylemma}{Lemma}

\newtheorem{myProp}{Proposition}

\usepackage{url}

\begin{document}

\title{QuantumFed: A Federated Learning Framework for Collaborative Quantum Training
}


\author{\IEEEauthorblockN{Qi Xia}
\IEEEauthorblockA{\textit{Department of Computer Science} \\
\textit{College of William and Mary}\\
Williamsburg, VA 23185, USA \\
qxia@cs.wm.edu}
\and
\IEEEauthorblockN{Qun Li}
\IEEEauthorblockA{\textit{Department of Computer Science} \\
\textit{College of William and Mary}\\
Williamsburg, VA 23185, USA \\
liqun@cs.wm.edu}
}

\maketitle

\begin{abstract}
    With the fast development of quantum computing and deep learning, quantum neural networks have attracted great attention recently. By leveraging the power of quantum computing,  deep neural networks can potentially overcome computational power limitations in classic machine learning. However, 
    when multiple quantum machines wish to train a global model using the local data on each machine, it may be very difficult to copy the data into one machine and train the model. Therefore, a collaborative quantum neural network framework is necessary. In this article, we borrow the core idea of federated learning to propose QuantumFed, a quantum federated learning framework to have multiple quantum nodes with local quantum data train a mode together. Our experiments show the feasibility and robustness of our framework.
\end{abstract}


\section{Introduction}

Quantum computing has been greatly developed in recent years. The idea of building a quantum computing based Turing machine was first proposed in early 1980s by Paul Benioff~\cite{Benioff1980}. Many explorations on this area were conducted by Richard Feynman~\cite{Feynman1982}, David Deutsch~\cite{1985RSPSA.400...97D}, etc. and people started to believe that quantum computing has the capability to beat classic computer in some tasks. In $1994$, Shor's algorithm was proposed to factor an integer using a quantum computer in polynomial time, which is exponentially faster than the fastest classic algorithms~\cite{doi:10.1137/S0036144598347011}. Recently, Google AI~\cite{Arute2019} and USTC~\cite{Zhong1460} claimed quantum supremacy for tasks that are infeasible on any classic computer.
In the meantime, deep neural network~\cite{LeCun2015} has been found efficient in many practical tasks such as computer vision~\cite{726791,Voulodimos2018,He2015,7298965}, natural language processing~\cite{10.1145/1390156.1390177,chen-etal-2017-enhanced,NIPS2017_3f5ee243,Devlin2018}, etc. It uses a hierarchical neural architecture to understand the world and achieves a great success in both industry implementation and academic research.

In recent years, there is a trend of combining deep learning and quantum computing together to reduce the huge computational cost for larger and deeper classic neural networks. For this purpose, the quantum neural network was naturally proposed~\cite{KAK1995143}. A quantum neural network uses the idea of a classical neural network in a quantum way to learn from the training data. By utilizing the main property of qubit superposition and entanglement in quantum mechanics, it tries to improve the computational efficiency and reduce the long training time and heavy computational resources in deep learning~\cite{GUPTA2001355,NIPS2003_50525975,Schuld2014,Wan2017,Beer2020}.

In deep neural network  model training, sometimes it is necessary to train a model through multiple machines in a distributed manner, e.g., federated learning~\cite{FEDERATED,FEDERATED2}. Federated learning is a collaborative way to train a global model where each node has their private local data in classic machine learning. For quantum machine learning, it is natural to use this collaborative approach for collaborative training. In this paper, we learn from this idea and propose a QuantumFed framework. Our contributions are summarized below:


\begin{itemize}
    \item We propose QuantumFed, a quantum federated learning framework to collaborate multiple quantum nodes with local quantum data together to train a global quantum neural network model.
    \item We conduct several simulation experiments to show that our QuantumFed framework is capable of collaborating multiple nodes and is robust for noisy data.
\end{itemize}

\section{Preliminaries}

\subsection{Quantum Computing Basis}

In quantum computing, the qubit is the basic unit to represent the information. A qubit has two basis states $|0\rangle$ and $|1\rangle$ like the classic bit in traditional computer, but it can also be in a superposition, which is a combination of the two basis states: $|\psi\rangle=\alpha|0\rangle+\beta|1\rangle$ where $\alpha^2+\beta^2=1$. Therefore, a qubit is capable of expressing more information than a classic bit. When observing the qubit, it will collapse to one of the basis states with corresponding probability, and thus we can get a statistically accurate estimation after sufficient times of observations. In addition, the entanglement of qubits allows more qubits to have correlations with each other and $n$ qubits, in this scenario, will have $2^n$ basis states and can be in a superposition among them, which carries an exponentially increasing amount of information.

In order to perform computations on the qubits, there are several common quantum logic gates: Pauli, Hadamard, Controlled Not. Unlike the AND and OR gate from classic computers, quantum operators are always reversible and will compute an output with the same dimension, and thus can be represented as a unitary. If we represent the input qubits state as a column vector, for example, $|\psi\rangle=\frac{1}{\sqrt{6}}|00\rangle+\frac{1}{\sqrt{6}}|01\rangle+\frac{1}{\sqrt{3}}|10\rangle+\frac{1}{\sqrt{3}}|11\rangle\rightarrow [\frac{1}{\sqrt{6}},\frac{1}{\sqrt{6}},\frac{1}{\sqrt{3}},\frac{1}{\sqrt{3}}]^{T}$, the output of operators are corresponding unitary left multiplied states.

\subsection{Quantum Neural Network}

There are lots of explorations of implementing deep neural networks in a quantum way. In this article, we adopt a widely used quantum deep neural network architecture as Figure~\ref{qnn}.
\begin{figure}[htb]
    \centering
    \includegraphics[width=0.71\columnwidth]{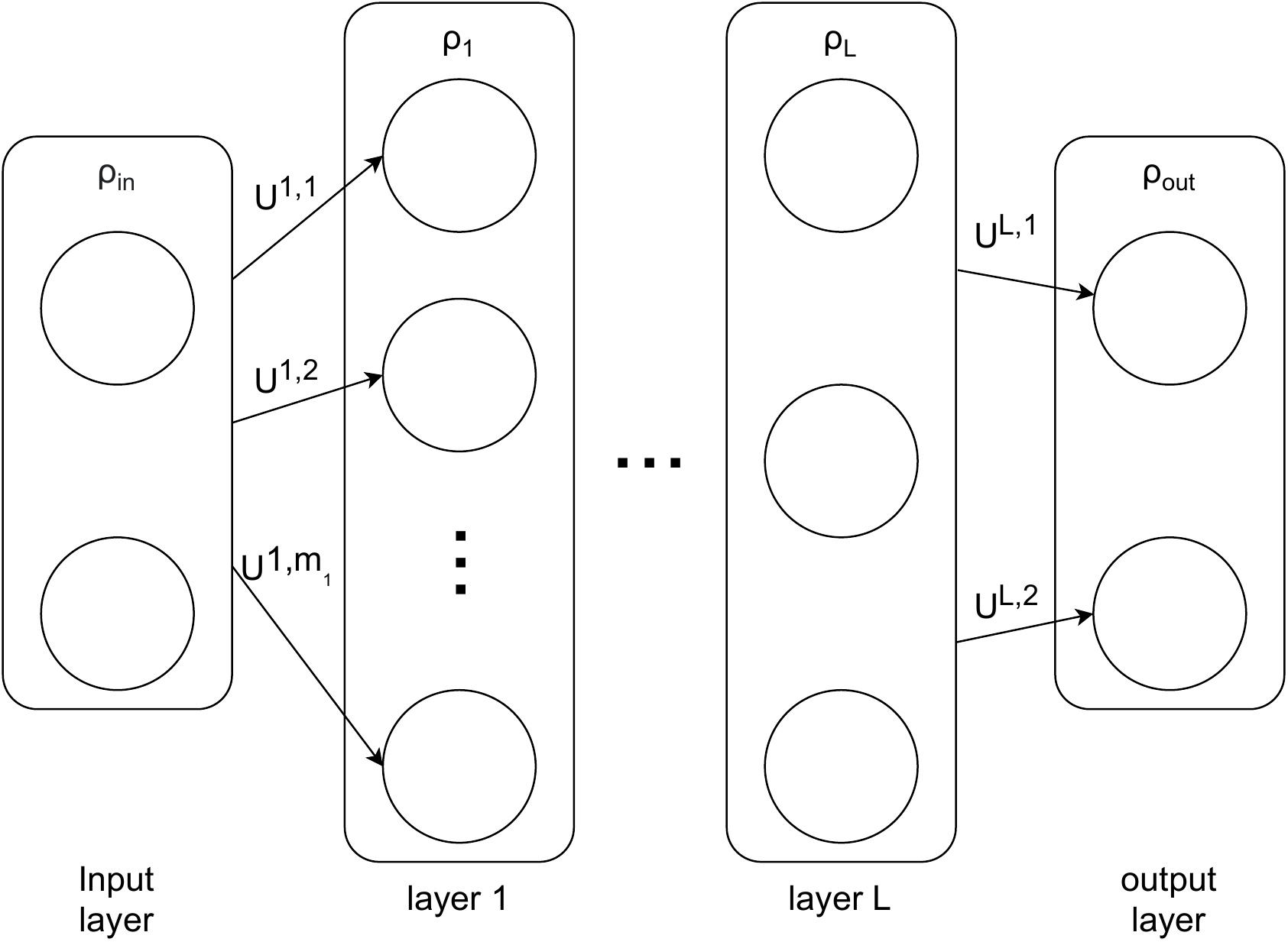}
    \caption{An architecture example of the quantum neural network.}
    \label{qnn}
\end{figure}
Assume in layer $l$, the input is a state $\rho^{l-1}$ of $m_{l-1}$ qubits and this layer will give an output of $m_l$ qubits, then the $l$-the layer transition map $\mathcal{E}^l$ is given by:
\begin{equation}
    \mathcal{E}^l(\rho^{l-1})=\text{tr}_{l-1}(U^l(\rho^{l-1}\otimes|0\cdots0\rangle_l\langle0\cdots0|){U^l}^{\dagger})
\end{equation}
$U^{l}$ is the $2^{m_{l-1}+m_l}\times2^{m_{l-1}+m_l}$ dimensional perceptron unitary of layer $l$. A partial trace operation is performed to get the output state of layer $l$. For simplicity, we apply $U^l$ by sequentially applying $m_l$ independent perceptron unitaries $U^{l,j}$ that act on $m_{l-1}$ input qubits and $j$-th qubit in layer $l$, that is, $U^l=\prod_{j=m_l}^1U^{l,j}$. Note that $U^{l,j}$ here are acting on the current layer, which means $U^{l,j}$ is actually $U^{l,j}\otimes\mathbb{I}^{l}_{1,\cdots j-1,j+1,\cdots m_l}$. In this way, we can feedforward the input state layer by layer to get an output state:
\begin{equation}
    \rho^{out}=\mathcal{E}^{out}(\mathcal{E}^L(\cdots\mathcal{E}^2(\mathcal{E}^1(\rho^{in}))\cdots))
\end{equation}
The hyper parameters of the quantum deep neural network are the unitaries, so as long as we have the network structure and unitaries, we can describe a model.

In order to represent the input and output data in a quantum way, for data that is stored by classic bits, we need to first transform the data to qubit representation. One way to do this is that we can use a $d$-qubits state $|\psi\rangle_d$ to represent a superposition of $2^d$ basis states in Hilbert space $\mathcal{H}^{2^d}$, that is, $|\psi\rangle_d=\sum_{i=1}^{2^d}\alpha_i|z_i\rangle$ where $\alpha_i$ is the complex amplitude and $z_i$ is a basis state in $\mathcal{H}^{2^d}$.

\subsection{Federated Learning}

Federated learning is a special kind of collaborative distributed learning in the machine learning area to train a global model across multiple computational devices or nodes who keep their private local training data.
In federated learning, the central server only keeps a global model and does not keep data. Each node queries a global model at some times, performs local updates on their local data, and uploads the update information to the central server. During one iteration, the central server receives all the computation results and updates the global model by aggregating local updates. In a classic neural network environment, the global model parameter at time $t$ is $w_t$. Assume there are $m$ participating nodes at time $t$ and node $i$ performs local update to get updated model $w_t^i=\text{LocalUpdate}(w_t)$. The LocalUpdate function is the classic gradient descent for one or more steps on local data. Then after the central server receives all the local updated models, it aggregates those models and updates the global model by $w_{t+1}=\text{GlobalUpdate}(w_t^1,w_t^2,\cdots,w_t^m)$.  In practice, we usually simply use a weighted average function as GlobalUpdate to aggregate the uploaded models, where the weight is usually the data volume of each node.

\section{QuantumFed Framework}

In this section, we will detail our QuantumFed framework.

\subsection{Cost Function}

The cost functions in classic neural networks are usually mean squared error loss or cross-entropy loss. Although these cost functions can still be applied in quantum neural networks, they are not easy to compute with quantum operations. Here like \cite{Beer2020}, we use fidelity as our cost function to measure the difference between label data states and output states. Fidelity represents the probability that one state will be identified as the other state in one measurement. Let $(\phi_x^{in},\phi_x^{out}),x=1,2,\cdots,N$ be the training data and $\rho_x^{out}$ be the output states that are derived by the current quantum neural network with the input data $\phi_x^{in}$, the cost function $\mathcal{C}$ is:
\begin{equation}
    \mathcal{C}=\frac{1}{N}\sum_{x=1}^N\langle\phi_x^{out}|\rho_x^{out}|\phi_x^{out}\rangle
    \label{fidelity}
\end{equation}
(\ref{fidelity}) measures the closeness between two states. When the output states are not pure, we can use a generalized fidelity function: $\mathcal{C}=\frac{1}{N}\sum_{x=1}^N(tr\sqrt{{\phi_x^{out}}^{1/2}\rho_x^{out}{\phi_x^{out}}^{1/2}})^2$. Note that the value of the fidelity cost function is between $0$ and $1$, while $1$ expresses the best performance.

\subsection{Local Update}

The local update is performed in each node in the quantum federated learning system. The goal of the local update is to maximize the cost function to $1$ based on the local dataset in the given steps (interval length). In quantum neural networks, the analogue of the model weight in classic neural networks is the model perceptron unitary $U$, and the model update is defined as $U\rightarrow e^{i\epsilon K}U$. Here $\epsilon$ is the update step size and $K$ is the update matrix. Therefore, each local step is to maximize the cost function by choosing an appropriate update matrix:
\begin{equation}
    K=\arg_{K}\max (\mathcal{C}(e^{i\epsilon K}U, (\phi_x^{in},\phi_x^{out}))-\lambda\|K\|_2^2)
    \label{defineK}
\end{equation}
$\mathcal{C}(U,(\phi_x^{in},\phi_x^{out}))$ is the fidelity cost of a model with perceptron unitary $U$ and local dataset $(\phi_x^{in},\phi_x^{out})$. The $-\lambda\|K\|_2$ is introduced by a Lagrange multiplier $\lambda$ to bound the norm of update matrix $K$ and $\|\cdot\|_2$ is the matrix $L_2$-norm.

Specifically, on the node $n$ side, let $(\phi_{n,x}^{in},\phi_{n,x}^{out}),x=1,2,\cdots,N_n$ be the local training data in node $n$, $U_n^{l,j}$ be the perceptron unitary of layer $l$, perceptron $j$ and $U_n^l=\prod_{j=m_l}^{1}U_{n}^{l,j}$, the output states of data $x$ at layer $l$ $\rho_{x}^l$ equal to $tr_{l-1}(U_n^l(\rho_{x}^{l-1}\otimes|0\cdots0\rangle_l\langle 0\cdots0|){U_n^l}^{\dagger})$ similar to \cite{Beer2020}, we can derive the update matrix $K_j^l$ by Proposition~\ref{computeK}.

\begin{myProp}
\label{computeK}
Let the cost function be the fidelity defined in (\ref{fidelity}), we can solve (\ref{defineK}) using gradient ascent by the following:
\begin{equation}
    K_j^l=\eta\frac{2^{m_{l-1}}i}{N_n}\sum_{x=1}^{N_n}tr_{rest}M_x^{l,j}
    \label{propK}
\end{equation}
$tr_{rest}$ is over all qubits that are not affected by $U_{n}^{l,j}$ and $M_x^{l,j}$ is computed by $M_x^{l,j}=[\prod_{\alpha=j}^1U_{n}^{l,\alpha}(\rho_{x}^{l-1}\otimes|0\cdots0\rangle_l\langle 0\cdots0|)\prod_{\alpha=1}^j{U_{n}^{l,\alpha}}^{\dagger},\prod_{\alpha=j+1}^{m_l}{U_{n}^{l,\alpha}}^{\dagger}(\mathbb{I}_{l-1}\otimes\sigma_{x}^l)\prod_{\alpha=m_l}^{j+1}U_{n}^{l,\alpha}]$. Here $\mathcal{F}^l$ is the adjoint channel to $\mathcal{E}^l$ and $\sigma_x^l=\mathcal{F}^{l+1}(\cdots\mathcal{F}^{out}(|\phi_{n,x}^{out}\rangle\langle\phi_{n,x}^{out}|))$.
\end{myProp}

From Proposition~\ref{computeK}, we can derive a closed-form update matrix for each perceptron unitary. This update matrix is an analogue of the gradient in the classic neural network, and the way to derive it is like the back-propagation process. Therefore, we can update the perceptron unitary like gradient descent based on local data and Proposition~\ref{computeK} for each step to maximize the cost function.

In classic federated learning, participating nodes are not required to do only one step gradient descent in each iteration. Therefore, we also assume that the local perceptron unitary can update for several steps. Here we define the number of steps as the interval length $I_l$. Then the local update algorithm QuanFedNode is described in Algorithm~\ref{alg1}.

\addtolength{\topmargin}{0.1in}

\begin{algorithm}
    \caption{QuanFedNode (Node $n$ Side)}
    \label{alg1}
    \begin{algorithmic}[1]
    \REQUIRE ~~\\
    Network architecture: there are $L$ layers in the quantum neural network and layer $l$ has $m_l$ quantum perceptrons;
    a copy of network perceptron unitaries from the global model $U_{n}^{l,j}=U_t^{l,j}$;
    training data: $(\phi_{n,x}^{in}, \phi_{n,x}^{out}),x=1,2,\cdots,N_n$;
    interval length $I_l$;
    total number of data among all participating nodes $N_t$;
    learning rate $\eta$ and update step size $\epsilon$;
    \ENSURE ~~\\
    Send update unitaries to the central server.
    \vspace{6pt}
    \STATE Set the interval index $k=1$;\\
    \STATE If $k<=I_l$, continue to the next step, otherwise go to step~$6$;\\
    \STATE Feedforward the training data at each layer:
    \begin{itemize}
        \item For every layer $l$, apply the current channel $\mathcal{E}^l$ to layer $l-1$: let $U_n^l=\prod_{j=m_l}^{1}U_{n}^{l,j}$;
        \item Let $\rho_{x}^l=tr_{l-1}(U_n^l(\rho_{x}^{l-1}\otimes|0\cdots0\rangle_l\langle 0\cdots0|){U_n^l}^{\dagger})$ and store $\rho_{x}^l$ for every layer;
    \end{itemize}
    \STATE Temporarily update the network:
    \begin{itemize}
        \item Compute the unitary update parameter at layer $l$, perceptron $j$, $K_j^l$ by (\ref{propK});
        \item Store update unitary at interval $k$, $U_{n,k}^{l,j}=e^{i\epsilon\frac{N_n}{N_t}K_j^l}$, and temporarily update the network by $U_{n}^{l,j}=e^{i\epsilon K_j^l}U_{n}^{l,j}$;
    \end{itemize}
    \STATE Let $k=k+1$ and go to step~$2$;\\
    \STATE Send all stored update unitaries $U_{n,k}^{l,j}$ to the central server.
    \end{algorithmic}
\end{algorithm}

In QuanFedNode algorithm, there are basically two steps:
\begin{itemize}
    \item Feedforward step. We apply the input state of the training data to the quantum neural network and feedforward it to every qubit by using the perceptron unitaries.
    \item Temporary update step. We first compute the unitary update matrix $K_j^l$ for layer $l$, perceptron $j$ by Proposition~\ref{computeK}. Then the local temporary update is derived by $U_{n}^{l,j}\leftarrow e^{i\epsilon K_j^l}U_n^{l,j}$. Meanwhile, we also compute another update unitary at interval $k$, $U_{n,k}^{l,j}=e^{i\epsilon\frac{N_n}{N_t}K_j^l}$. $N_t$ is the number of data on all participating nodes in this iteration. $U_{n,k}^{l,j}$ is computed for global update and will be sent to the central server later.
\end{itemize}
The temporary update will be processed in each participating node with their local training data and repeat for $I_l$ times. Then we can simply send the update unitaries $U_{n,k}^{l,j}$ to the central server.

\subsection{Global Update}
\label{sec-global-update}

Global update is performed on the central server side. It maintains a global model that is updated by each node's local data and update unitaries. The goal of the global update is to maximize the cost function based on the global dataset among all quantum nodes. Because the data is stored in each node and the central server is not able to access the private local data, the global update can only be computed based on the update unitaries that are uploaded by each node. We describe the QuanFedPS algorithm in Algorithm~\ref{alg2}.

\begin{algorithm}
    \caption{QuanFedPS (Central Server Side)}
    \label{alg2}
    \begin{algorithmic}[1]
    \REQUIRE ~~\\
    Network architecture: there are $L$ layers in the quantum neural network and layer $l$ has $m_l$ quantum perceptrons;
    total number of nodes $N$ and number of selected nodes in each iteration $N_p$;
    total synchronization iterations $N_s$;
    number of training data on node $n$ $N_n$;
    interval length $I_l$;
    \ENSURE ~~\\
    The trained quantum neural network.
    \vspace{6pt}
    \STATE Initialize the network by randomly choosing all the unitaries $U_t^{l,j}$, set the iteration index $t=1$;
    \STATE If $t<=N_s$, continue to the next step, otherwise go to step~$6$;
    \STATE Randomly select $N_p$ nodes from all nodes. Assume the set of selected node indexes is $S_n$, compute the total number of data among all participating nodes $N_t=\sum_{n\in S_n}N_n$. For each selected node, run QuanFedNode algorithm and get update unitaries $U_{n,k}^{l,j}$;
    \STATE Compute the global update unitaries by applying update unitaries from all selected nodes:
    \begin{equation}
        U^{l,j}=\prod_{k=I_l}^{1}\prod_{n\in S_n}U_{n,k}^{l,j}
        \label{global-update}
    \end{equation}
    \STATE Update the global model by $U_{t+1}^{l,j}=U^{l,j}U_t^{l,j}$;
    \STATE Let $t=t+1$ and go to step~$2$;
    \STATE Output the trained quantum neural network model.
    \end{algorithmic}
\end{algorithm}

Basically there are three major steps in QuanFedPS algorithm:
\begin{itemize}
    \item Initialization step. At the beginning of the quantum federated training process, the central server first initializes the model parameters (perceptron unitaries) by randomly assigning the value.
    \item Node selection step. Like the classic federated learning framework, we need to randomly select $N_p$ nodes out of all $N$ nodes who will participate in the current iteration. This can help improve the randomness of the data distribution and decrease the data heterogeneity. Besides, it can reduce the communication cost by selecting fewer nodes.
    \item Global update step. After participating nodes complete the local training and send update unitaries back to the central server, the central server updates the global model by applying those update unitaries and finishing the current iteration. We will take the global model update for $N_s$ iterations.
\end{itemize}

The design of global update is based on the observation that the order of applying update unitaries almost does not matter and the update unitaries almost surely have multiplicative identity property when $\epsilon\rightarrow 0$. Theoretically, we have the following lemma.

\begin{mylemma}
    Assume $U_1=e^{i\epsilon K_1}, U_2=e^{i\epsilon K_2}$ are two update unitaries and $K_1, K_2$ are bounded by the $L_2$-norm, we have $\lim_{\epsilon\rightarrow 0}U_1U_2=e^{i\epsilon(K_1+K_2)}$ at convergence speed $O(\epsilon^2)$.
    \label{multiplicative}
\end{mylemma}

\begin{proof}
    By Taylor's expansion, we have:
    \begin{align*}
        U_1 &= I+i\epsilon K_1 + O(\epsilon^2)\\
        U_2 &= I+i\epsilon K_2 + O(\epsilon^2)\\
        e^{i\epsilon(K_1+K_2)} &= I+i\epsilon(K_1+K_2) + O(\epsilon^2)
    \end{align*}
    Then we have:
    \begin{equation}
        U_1U_2-e^{i\epsilon(K_1+K_2)} = O(\epsilon^2)
        \label{eqsmall}
    \end{equation}
    From (\ref{eqsmall}), we can derive $\lim_{\epsilon\rightarrow 0}U_1U_2=e^{i\epsilon(K_1+K_2)}$ at convergence speed $O(\epsilon^2)$.
\end{proof}

From Lemma~\ref{multiplicative}, when $\epsilon$ is small enough, we can rewrite the global update unitaries (\ref{global-update}) as:
\begin{equation}
    U^{l,j}=\prod_{k=I_l}^{1}e^{i\epsilon K_k^{l,j}}, K_k^{l,j}=\frac{\sum_{n\in S_n}N_nK_{n,k}^{l,j}}{\sum_{n\in S_n}N_n}
\end{equation}
Here we define $K_{n,k}^{l,j}$ as the update matrix for node $n$ at step $k$, layer $l$, and perceptron $j$. Note that the update matrix $K_{n,k}^{l,j}$ is derived from (\ref{propK}), and (\ref{propK}) is actually an average of the partial trace of $M_{x}^{l,j}$ for all the local data. Therefore, if we denote the dataset on node $n$ is $D_n$, and $D_p=\bigcup_{n\in S_n}D_n$, we have:
\begin{align*}
    K_k^{l,j} &= \frac{\sum_{n\in S_n}(N_n\eta\frac{2^{m_l-1}i}{N_n}\sum_{x\in D_n}tr_{rest}M_x^{l,j})}{\sum_{n\in S_n}N_n}\\
    &= \eta\frac{2^{m_l-1}i}{\sum_{n\in S_n}N_n}\sum_{x\in D_p}tr_{rest}M_x^{l,j}
\end{align*}
This is equivalent to compute for a local update on the union dataset of the data on all participating nodes in this iteration when $k=1$. So when the interval length is set to $1$, the QuantumFed framework is exactly the same as training on a single quantum machine with all data. However, when the interval length is greater than $1$, things become much more complicated. We will discuss this problem in the next subsection.

\subsection{Discussions}

\subsubsection{GD vs SGD}

The QuanFedNode algorithm that we described in Algorithm~\ref{alg1} uses all the training data in each update step. Therefore, the training process is more like gradient descent (GD) in classic deep learning.
An alternative training method is by randomly choosing a mini-batch of training data in each step, which is an analogue of mini-batch stochastic gradient descent (SGD) in classic deep learning. SGD can solve the biased data distribution problem by introducing randomness and reduce the computational cost. In the experiment part, we will compare GD-type and SGD-type quantum federated training.

\subsubsection{Interval Length $> 1$}

As we discussed before, when $I_l=1$, the QuantumFed framework is exactly the same as training on a single machine. When $I_l>1$, because of the local temporary update, the update unitaries that are computed after the first step are based on the temporarily updated model parameters, which are not the same as each other participating nodes in one iteration. One way to understand it is that since we usually choose a small $\epsilon$, the temporary update is a small perturbation of the perceptron unitaries. Assume the perceptron unitary is $U_p$ and the update unitary is $U_u=e^{i\epsilon K_u}$. By Taylor's expansion, the update perturbation is given by:
\begin{align*}
    U_uU_p-U_p &= (I+i\epsilon K_u+O(\epsilon^2)-I)U_p\\
    &= (i\epsilon K_u+O(\epsilon^2))U_p
\end{align*}
Therefore, the perturbation is small compared to the perceptron unitary when $\epsilon$ is small and $K_u$ is bounded, and intuitively we can consider the local temporary updated model as a same model. We also show it is feasible to use larger interval length in the experiment, where it reduces the synchronization iterations and accelerates the training speed.



\subsubsection{Why Unitary}

When communicating between the central federated server and quantum nodes, we choose to use update unitary matrix as the model parameters for transmission. The reasons are below. First, in classic federated learning, each node sends the local updated model, or gradient to the central server. In the quantum neural network, the analogue is the local updated network unitary $U_n^{l,j}$ or updated unitary $U_{n,k}^{l,j}$. Second, the central server can simply apply the update unitaries from each worker to update the global model. It is more convenient and reduces the computations in the central server side. Third, unitary is the basic operation in quantum mechanics and it is easier to optimize in the system level.


\subsubsection{Learning Rate and Step Size}

In the QuanFedNode algorithm, there are learning rate $\eta$ and step size $\epsilon$ to control the update and they have different meanings. $\epsilon$ is derived from classic quantum neural network update, and $\eta$ is derived from how we would like to bound the update matrix $K$. A larger $\eta$ leads to a tighter bound of $K$. However, actually we can rewrite the update unitary by:
\begin{equation}
    U_{n,k}^{l,j}=e^{i\epsilon\frac{N_n}{N_t}K_j^l}=e^{i\epsilon\eta\frac{2^{m_{l-1}}i}{N_t}\sum_{x=1}^{N_n}tr_{rest}M_x^{l,j}}
\end{equation}
We can actually combine $\epsilon$ and $\eta$ to one parameter.
For convenience and easy to understand, we set $\lambda=1.0$ at all time and fine tune the step size by adjusting $\epsilon$ in practice.




\section{Experiment}

In this section, we conduct simulated experiments for our QuantumFed framework. 

\subsection{Environment Setup}

Here we use a quantum environment simulated by QuTip library~\footnote{QuTip: https://github.com/qutip/qutip} (Quantum Toolbox in Python). We set up our experiment environment in the following aspects. 

First, in order to get the training data, similar to~\cite{Beer2020}, we first randomly generate a global unitary $U_g$ which is the unitary we would like to approximate. Then we randomly generate the training data input and apply the global unitary to the input to get the corresponding output. We use the randomly generated input and output pair as the clean training data. The same method is applied to generate the test data. In this way, we can generate clean training data $(|\phi_{n,x}^{in}\rangle,U_g|\phi_{n,x}^{in}\rangle)$ on the node $n$ side, and test data $(|\phi_{test,x}^{in}\rangle,U_g|\phi_{test,x}^{in}\rangle)$ on the central server side. In order to show the robustness of the training, we also pollute a proportion of training data with randomly generated input and output to get noisy training data.
Second, as for the quantum neural network architecture, because the experiments that we conduct are in a simulated environment using the classic computer and the computational complexity increases exponentially with the width of the network increases, we choose to train small-size quantum neural networks with width that are not greater than $3$. In this section, we choose a network of size $2$-$3$-$2$.
Third, in order to simulate the heterogeneous federated learning environment, we put similar training data into the same node. We first gather all the generated training data from all nodes, sort them by their vector representation value, and divide them to each node in order. In this way, we can somehow guarantee that the data on each node is not i.i.d.
Fourth, we measure the experiment results using two metrics. First metric is the fidelity cost function that we defined in (\ref{fidelity}), to show the probability that the output state will be identified as the output label in a measurement. We also adopt another metric mean square error (MSE) that is widely used in classic machine learning as a comparison. The MSE is defined below:
\begin{equation}
    \text{MSE}=\frac{1}{N}\sum_{x=1}^{N}\|\rho_x^{out}-|\phi_x^{out}\rangle\langle\phi_x^{out}|\|^2
\end{equation}
We examine our experiments using both metrics on the training and test data respectively to show the performance.

We set $\eta=1.0$, $\epsilon=0.1$, $N=100$, and $N_p=10$ if not specified.

\subsection{Experiment Results}

\subsubsection{Accuracy}

We first show how a $2$-$3$-$2$ quantum network performs with different interval lengths in Figure~\ref{interval_comparison}.
\begin{figure}[htb]
    \centering
    \includegraphics[width=\columnwidth]{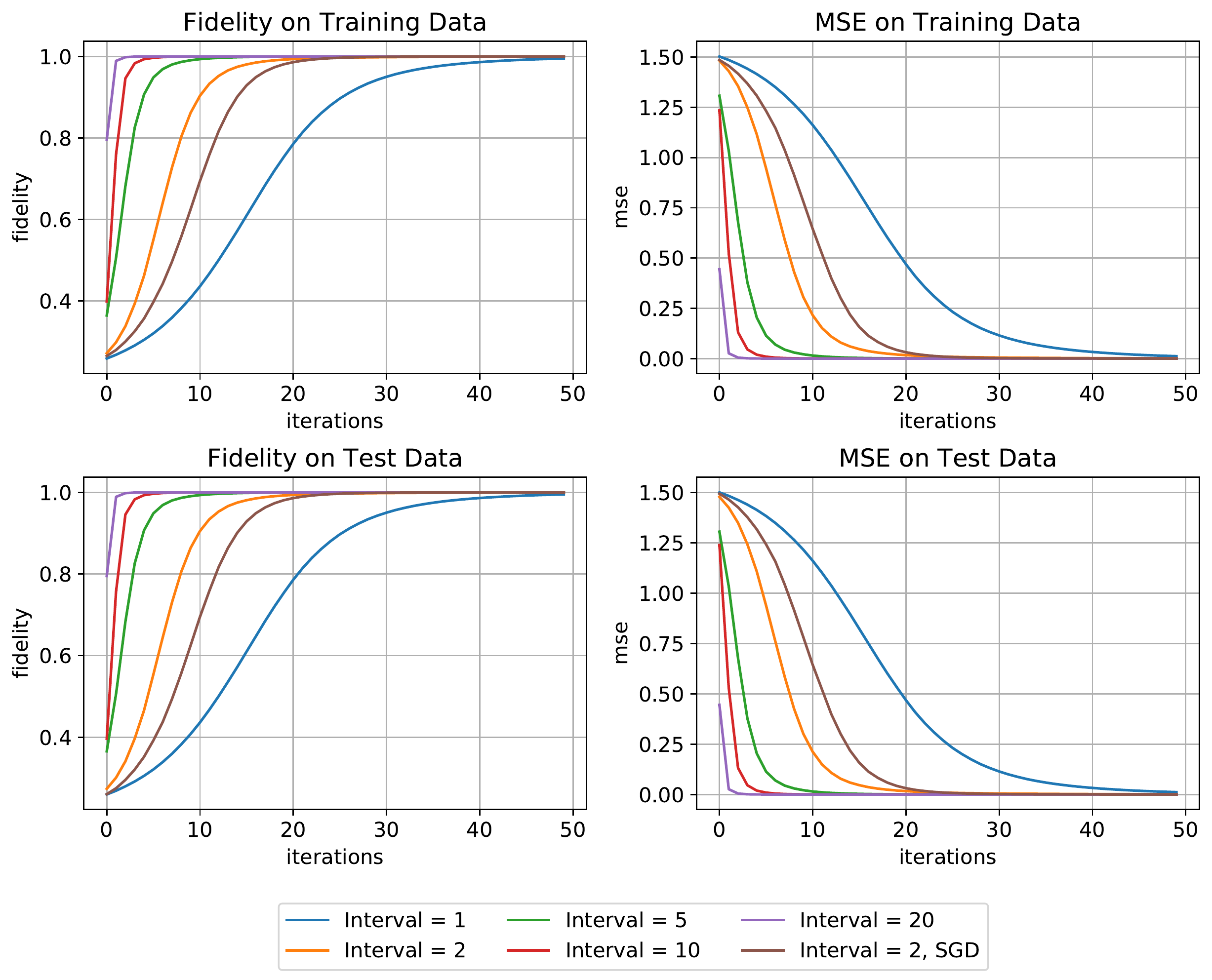}
    \caption{Experiment results of a $2$-$3$-$2$ quantum network with different interval lengths.}
    \label{interval_comparison}
\end{figure}
Here, the interval length of $1$ case is actually the same as the scenario that is running on a single machine. Therefore, we can see that after $50$ iterations, all of them reach fidelity of approximately $1$ and MSE of approximately $0$ on both training data and test data. This shows that our QuantumFed framework works on collaborating different quantum nodes for training a global model. Besides, we can find that the performance becomes better when we conduct more local steps in each iteration. This is because we have more local training on local data, which learns more information in each iteration. In addition, as a comparison, we also plot the SGD scenario with interval length $2$, here we use a mini-batch of $5$ for this experiment. We can see that the convergence speed is slower a little bit, which makes sense because we have less data in each iteration, but the final performance is similar. This shows that our framework is feasible for both SGD and GD optimization.

\subsubsection{Robustness}
We then show the performance of a $2$-$3$-$2$ quantum network with a different ratio of noisy data. We compare the data with $10\%$ noisy data to $90\%$ noisy data on noisy training data and clean test data. The results are in Figure~\ref{noise_comparison}.
\begin{figure}[htb]
    \centering
    \includegraphics[width=\columnwidth]{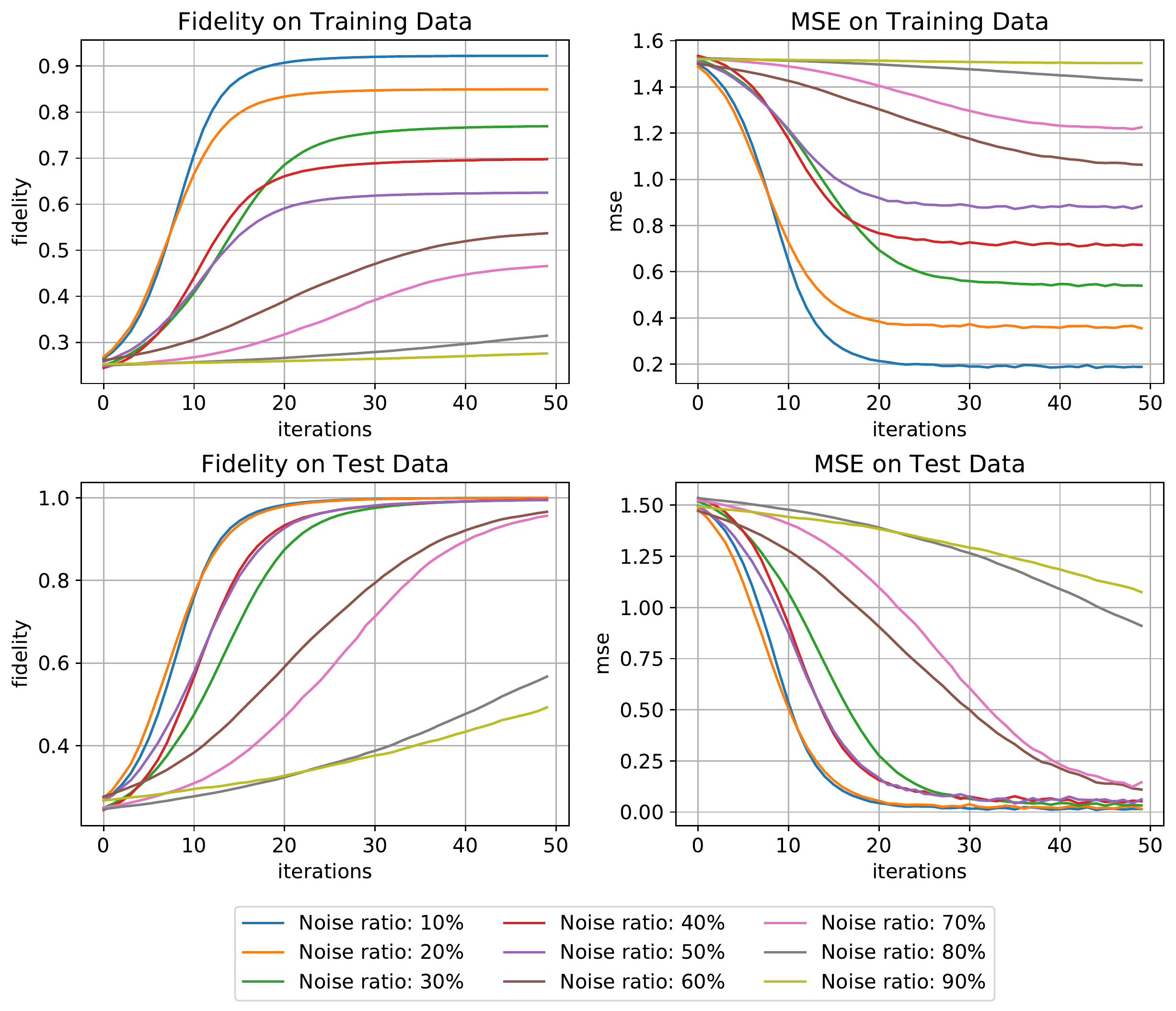}
    \caption{Experiment results of a $2$-$3$-$2$ quantum network with different ratios of noisy data.}
    \label{noise_comparison}
\end{figure}
As we can see from the figure, the performance keeps acceptable when the noise data ratio is no more than $70\%$, while the final performance is similar when the noise ratio is no more than $50\%$. This shows the robustness of our QuantumFed framework and it is able to resist with a considerable proportion of noisy data. 


\section{Conclusion}

In this paper, we propose a novel quantum federated learning framework in which multiple quantum nodes collaborate using local quantum data. Several experiments are conducted to show the feasibility and robustness of our QuantumFed framework. With the emergence of quantum computing, the potential of quantum neural networks is enormous. We believe it will be practical to train a deep neural network collaboratively on multiple quantum devices in the near future.




\bibliographystyle{IEEEtranS}
\bibliography{main}

\end{document}